\documentclass[conference]{IEEEtran}
\IEEEoverridecommandlockouts
\usepackage{cite}
\usepackage{amsmath,amssymb,amsfonts}
\usepackage{algorithmic}
\usepackage{graphicx}
\usepackage{textcomp}
\usepackage{xcolor}
\def\BibTeX{{\rm B\kern-.05em{\sc i\kern-.025em b}\kern-.08em
    T\kern-.1667em\lower.7ex\hbox{E}\kern-.125emX}}
    
\usepackage[square,numbers]{natbib}
\usepackage{algorithm}
\usepackage{algorithmic}
\usepackage{subfigure}
\usepackage{booktabs} 
\usepackage{amsthm}

\usepackage{hyperref}

\newtheorem{theorem}{Theorem}[section]

\newtheorem{lemma}[theorem]{Lemma}

\newtheorem{assumption}[theorem]{Assumption}
\newtheorem{remark}[theorem]{Remark}

\begin{document}

\newcommand{\regret}{R}
\newcommand{\horizon}{n}
\newcommand{\horizonn}{n_1}
\newcommand{\horizonnn}{n_2}
\newcommand{\reward}{X}
\newcommand{\vecreward}{\Bar{X}}
\newcommand{\actset}{\mathcal{A}}
\newcommand{\Action}{A}
\newcommand{\action}{a}
\newcommand{\timeidx}{t}
\newcommand{\timeidxx}{s}
\newcommand{\totalreward}{S}
\newcommand{\policy}{\pi}
\newcommand{\dumnpolicy}{\pi^{\prime}}
\newcommand{\armscount}{N}
\newcommand{\rewardof}{r}
\newcommand{\numofactions}{K}
\newcommand{\armval}{\mu}
\newcommand{\bestarm}{\mu^{\ast}}
\newcommand{\bestarmidx}{a^{\ast}}
\newcommand{\suboptgap}{\Delta}
\newcommand{\numbatches}{M}
\newcommand{\batchsize}{b}
\newcommand{\grid}{\mathcal{T}}
\newcommand{\batchidx}{j}
\newcommand{\batchidxoftimeidx}{\batchidx(\timeidx)}
\newcommand{\batchpolicy}{\policy^{\batchsize}}
\newcommand{\history}{H}
\newcommand{\realnumbers}{\mathbb{R}}
\newcommand{\historyset}{\mathcal{H}}
\newcommand{\setofdist}[1]{\mathcal{M}_1{#1}}
\newcommand{\probmeasure}{\mathbb{P}}
\newcommand{\env}{\nu}
\newcommand{\avpolicy}[1]{\Bar{\policy}_#1}
\newcommand{\avbatchpolicy}[1]{\Bar{\policy}^{\batchsize}_#1}
\newcommand{\avdumnpolicy}[1]{\Bar{\policy}^{\prime}_#1}
\newcommand{\norm}[1]{\left\lVert#1\right\rVert}

\title{The Impact of Batch Learning in \\ Stochastic Linear Bandits \\
}

\author{\IEEEauthorblockN{Danil Provodin\textsuperscript{1,2}, Pratik Gajane\textsuperscript{1}, Mykola Pechenizkiy\textsuperscript{1,4}, Maurits Kaptein\textsuperscript{2,3}}
\IEEEauthorblockA{
\textsuperscript{1}\textit{Eindhoven University of Technology} Eindhoven, The Netherlands\\
\textsuperscript{2}\textit{Jheronimus Academy of Data Science}, ‘s-Hertogenbosch, The Netherlands\\
\textsuperscript{3}\textit{Tilburg University}, Tilburg, The Netherlands\\
\textsuperscript{4}\textit{University of Jyväskylä}, Jyväskylä, Finland\\
\texttt{ \{d.provodin,p.gajane,m.pechenizkiy\}@tue.nl, M.C.Kaptein@tilburguniversity.edu}}
}

\maketitle

\begin{abstract}
We consider a special case of bandit problems, named batched bandits, in which an agent observes batches of responses over a certain time period. Unlike previous work, we consider a more practically relevant batch-centric scenario of batch learning. That is to say, we provide a policy-agnostic regret analysis and demonstrate upper and lower bounds for the regret of a candidate policy. Our main theoretical results show that the impact of batch learning is a multiplicative factor of batch size relative to the regret of online behavior. Primarily, we study two settings of the stochastic linear bandits: bandits with finitely and infinitely many arms. While the regret bounds are the same for both settings, the former setting results hold under milder assumptions. Also, we provide a more robust result for the 2-armed bandit problem as an important insight. Finally, we demonstrate the consistency of theoretical results by conducting empirical experiments and reflect on optimal batch size choice.
\end{abstract}


\section{Introduction}
The stochastic bandit problem is one of the central topics of modern literature on sequential decision making, which aims to determine policies that maximize the expected reward. These policies are often learned either \textit{online} (sequentially) (see, e.g., \citep{LinUCB_2010, NIPS2017_28dd2c79, dimakopoulou2018estimation}) or \textit{offline} (statically) (see, e.g., \citep{swaminathan15a, zhou2018offline, joachims2018deep, athey2020policy}. In online problems, the agent learns through sequential interaction with the environment, adjusting the behavior for every single response. In offline learning, on the other hand, the agent learns from fixed historical data without the possibility to interact with the environment. Therefore, the agent's goal is to maximally exploit the static data to determine the best policy. However, in many application domains, batched feedback is an intrinsic characteristic of the problem, and neither setting provides a close approximation of the underlying reality \citep{marketing_DP, NIPS2011_e53a0a29, marketing_bandit, Hill_2017}. While the offline setting is not conducive to sequential learning, online learning is often curtailed by the limitations of practical applications. For example, in recommender systems and ad placement engines, treating users one at a time can become a formidable computational and/or engineering burden; in online marketing and clinical trials, environments design (campaigns/trials) and the presence of delayed feedback result in treating patients/customers organized into groups. In all of these applied cases, it is infeasible to learn truly one by one due to computational complexity or the impact of delay. In such cases, an online approach that acts on groups of observations is much more appealing from a practical point of view.

Because of the practical restrictions described above, we consider \textit{sequential batch learning} in bandit problems -- sequential interaction with the environment when responses are grouped in batches and observed by the agent only at the end of each batch. Generally speaking, sequential batch learning is a more generalized way of learning which covers both offline and online settings as special cases bringing together their advantages:
\begin{itemize}
    \item  Unlike offline learning, sequential batch learning retains the sequential nature of the problem.
    \item Unlike online learning, it is often appealing to implement batch learning in large-scale bandit problems as
        \begin{itemize}
            \item it does not require much engineering effort and resources, as experimental control over the system is unessential; and
            \item it does not need resources to shorten the feedback loop in time-delayed bandit problems.
        \end{itemize}
\end{itemize}

Unfortunately, a comprehensive understanding of the effects of the batch setting is still missing. Withdrawing the assumptions of online learning that have dominated much of the bandit literature raises fundamental questions as how to benchmark performance of candidate policies, and how one should choose the batch size for a given policy in order to achieve the rate-optimal regret bounds. As a consequence, it is now frequently a case in practice when the batch size is chosen for the purpose of computational accessibility rather than statistical evidence \citep{NIPS2011_e53a0a29, Hill_2017}. Moreover, while the asymptotics of batched policies is known (see, e.g., \citep{Auer_2010, cesabianchi2013online}),
the relatively small horizon performance of batch policies is much less understood while simultaneously being much more practically relevant. Thus, in this work, we make a significant step in these directions by providing a systematic study of the sequential batch learning problem.

In this work, we focus on the batch learning problem with respect to the batch size $b$. That is to say, we provide upper and lower regret bounds of a batch policy relative to its online behavior as a function of $b$; specifically, we establish that the impact of the batch learning is a multiplicative factor of batch size. The distinctive feature of our work, relative to previous batch learning literature, is that we provide a policy-agnostic analysis, which holds for a certain set of policies.

\noindent
In summary, this paper makes the following contributions \footnote{Some preliminary results have been presented at NeurIPS 2021 Workshop on Ecological Theory of Reinforcement Learning \citep{provodin2021impact}.}.:

\begin{enumerate}
    \item we formulate a more practically relevant batch-centric problem (Section \ref{sec:prob_formulation}) and establish upper and lower bounds on the performance for an arbitrary candidate policy (Section \ref{sec:warm-up});
    
    \item we demonstrate the validity of the theoretical results experimentally and reflect on the optimal batch size choice (Section \ref{empirical});
    
    \item we provide insight and guidance to the development of novel batch policies (Section \ref{sec:conclusion}). 
\end{enumerate}

\subsection{Related work}

Our setting lies in the intersection of batched bandits and bandits with delayed feedback. The origin of the batched bandit constituent can be traced back to \cite{Perchet_2016}, which proposes an explicit batched algorithm based on explore-then-commit policy for a two-armed batch bandit problem and explore its upper and lower regret bounds, giving rise to a rich line of work \cite{gao2019batched, han2020sequential, esfandiari2020regret, batchedTS}. However, regret bounds in prior batched bandit literature are incomparable to ours because those are with respect to the number of batches $M$, whereas we consider regret bounds with respect to batch size $b$.

The problem of batched bandits also relates to learning from delayed feedback (see, e.g., \citep{joulani2013online, vernade2017stochastic, pikeburke2018bandits, Zhou_context_delay_2019, Vernade2020}). The delayed setting focuses on similar regret bounds as we do in our study -- with respect to the size of delay $d$ (which is comparable with the batch size $b$). The closest works to ours are \cite{joulani2013online, Zhou_context_delay_2019, Vernade2020}. Similar to our work, reference \cite{joulani2013online} provides a policy-agnostic analysis of stochastic bandits. Specifically, they establish that delay increases regret in an additive way in stochastic problems. However, we provide a more general analysis of stochastic linear bandits in our study. References \cite{Zhou_context_delay_2019, Vernade2020} show that in linear bandits, the increase in regret due to rewards being delayed is a multiplicative $\sqrt{d}$ and $d$ factor, respectively. From the mathematical point of view, the main difference with our approach is that they propose a specific policy and analyze it, whereas we provide a more general policy-agnostic analysis. In that sense, their problem is easier than ours because policy behavior is completely known; yet, we make stronger assumptions on the batch size $b$.

\section{Problem formulation}
\label{sec:prob_formulation}
We consider a general model of \textit{sequential batch learning} in which, unlike classical online learning, we assume that the responses are grouped in batches and observed only at the end of each batch. For building up the formal notion of the batch learning setting in Section \ref{sec:batch}, we first introduce the basic notions of the classical setting in Section \ref{sec:lin_bandits}.

\subsection{Decision Procedure and Reward Structures}
\label{sec:lin_bandits}

We consider a general setting of the stochastic linear bandit problem. In this problem, the decision-maker (agent) has to make a sequence of decisions, possibly based on some side information, and for each decision it incurs a stochastic, although not necessarily immediately observed, reward. \footnote{We borrow notations and problem formulation for online setting from \cite{lattimore_szepesvari_2020}.} More formally, given a decision set $\actset_{\timeidx} \subset \mathbb{R}^d$ (for $d<\infty$) and a time horizon $\horizon < \infty$, at each time step $\timeidx \in \{1,2, \dots, \horizon\}$, the agent chooses an action $\Action_{\timeidx} \in \actset_{\timeidx}$  and reveals reward 
\begin{equation*}
    \reward_{\timeidx} = \langle \theta_*, \Action_{\timeidx} \rangle + \eta_{\timeidx}
\end{equation*}
where $\eta_{\timeidx}$ is the noise, and $\theta_*$ is the instance unknown for the agent. For simplicity, we assume each $\eta_{\timeidx}$ is 1-subGaussian: $\mathbb{E} [e^{\lambda \eta_{\timeidx}}] \leq e^{\lambda^2 / 2}, \forall \timeidx \in [n], \forall \lambda \in \mathbb{R}$. Although restrictive, this is a classical assumption for bandit literature, as it covers most reasonable noise distributions. Further, without loss of generality (since $d <  \infty$), we assume $\norm{\theta_*}_2 \leq 1$.

The goal of the agent is to maximize the total reward $\totalreward_{\horizon} = \sum_{\timeidx=1}^{\horizon} \reward_{\timeidx}$. To assess the performance of a policy $\policy$, we consider regret -- difference between the agent's total reward and cumulative reward obtained by an optimal policy:

\begin{equation*}
    \regret_\horizon (\policy) = \mathbb{E} \left [ \sum_{\timeidx=1}^{\horizon} \max_{\action \in \actset_{\timeidx}} \langle \theta_*, \action \rangle - \totalreward_{\horizon} \right ].
\end{equation*}

The generality of this setting lies in the choice of the decision set $\actset_{\timeidx}$: different choices of $\actset_{\timeidx}$ lead to different settings.

\paragraph{Stochastic bandit.} If $(e)_i$ are the unit vectors and $\actset_{\timeidx} = \{e_1,...,e_\numofactions \}$ then the resulting stochastic linear bandit problem reduces to the finite-armed stochastic setting. We can think of the problem as an interaction between the agent and environment $\env = (P_\action : \action \in \actset)$, where $\actset:=\{1, ..., \numofactions \}$ and $P_\action$ is the distribution of rewards for action $\action$. At each time step $\timeidx$, the agent chooses an action $\Action_{\timeidx} \in \actset$ and receives a reward $\reward_{\timeidx} \sim P_{\Action_\timeidx}$. In this case $\theta_* = (\armval_\action )_{\action}$, where $\armval_\action = \int_{\realnumbers} x d P_\action(x)$ is the expected reward of action $\action \in \actset$, and regret is the difference between the agents' total reward after $\horizon$ rounds and the best reward possible given a strategy of choosing single action:

\begin{equation}
\label{eq:regret definition}
    \regret_\horizon (\policy) = \horizon \max_{\action \in \actset} \armval_\action - \mathbb{E} \big [ \totalreward_{\horizon} \big ].
\end{equation}

\paragraph{Contextual linear bandit.}

If $\actset_{\timeidx} = \{ \psi(C_t, i): i \in [K] \}$, where $\psi: \mathcal{C} \times [K] \xrightarrow{} \mathbb{R}$ is a feature map and $C_t \in \mathcal{C}$ is a context observed by the agent at timestep $\timeidx$, then we have a contextual linear bandit. To illustrate how the above quantities arise in practice we give an example. Suppose, at the beginning of round $\timeidx$, a customer characterized by context $C_t$ visits a website selling books. When the agent applies action $\Action_\timeidx$ (recommends a book) to the customer, a reward $r_t(C_t, \Action_\timeidx)$ is obtained. The contextual linear bandit settings assumes that the expected reward $\mathbb{E}[r_t(c, \action)]$ is a parametrized mean reward function $g_{\theta}(\cdot)$ of feature vector $\psi(\cdot, \cdot)$, formally, $\mathbb{E}[r_t(c, \action) \text{ } | \text{ } C_t, \Action_\timeidx] = g_{\theta}(\psi(C_t, \Action_\timeidx))$. Note that the linear contextual bandits setting then corresponds to $g_{\theta}(\cdot) = \langle \theta, \cdot \rangle$.


\subsection{Sequential batch learning}
\label{sec:batch}
In the standard online learning setting, the decision maker immediately observes the reward $\reward_\timeidx$ after selecting action $\Action_\timeidx$ at time step $\timeidx$.  Consequently, in selecting $\Action_{\timeidx+1}$, the decision maker can base his decision on the current decision set $\actset_{\timeidx+1}$ and past history. Here, the past at time step $\timeidx$ is defined as

\begin{equation*}
    \history_{\timeidx} = (\Action_1, \reward_1, ... , \Action_{\timeidx}, \reward_{\timeidx}) \in 
    \realnumbers^{(d+1)\timeidx} \equiv \historyset_{\timeidx}
\end{equation*}
which is the sequence of action-reward pairs leading up to the state of the process at the previous time step $\timeidx$. Note that $\history_0 = \emptyset$. Let $\setofdist(X)$ be a set of all probability distributions over a set $X$. As such, a policy is a finite sequence $\policy = (\policy_{\timeidx})_{1 \leq \timeidx \leq \horizon}$ of maps of histories to distributions over actions (decision rules), formally, $\policy_\timeidx : \historyset_{\timeidx-1} \xrightarrow{} \setofdist(\actset_{\timeidx})$. Intuitively, following a policy $\policy$ means that in timestep $\timeidx$, the distribution of the action $\Action_\timeidx$ to be chosen for that timestep is $\policy_{\timeidx}(\history_{\timeidx-1})$: the probability that $\Action_{\timeidx} = a$ is $\policy_{\timeidx}(\action | \history_{\timeidx-1})$. Thus, when following a policy $\policy$, in timestep $\timeidx$ we get that

\begin{equation*}
    \probmeasure(\Action_{\timeidx} = \action | \history_{\timeidx-1}) = \policy_{\timeidx}(\action | \history_{\timeidx-1}).
\end{equation*}

In contrast to conventional approaches that require the reward to be observable after each choice of the action, our setting assumes only that rewards are released at specific predefined timesteps. Denote by $\grid = \timeidx_1,...,\timeidx_{\numbatches}$ a grid, which is a division of the time horizon $\horizon$ to $\numbatches$ batches of equal size $\batchsize$,  $1 = \timeidx_1 < ... < \timeidx_{\numbatches} = \horizon, \timeidx_{\batchidx} - \timeidx_{\batchidx-1} = \batchsize$ for all $\batchidx=1,..., \numbatches$. Without loss of generality we assume that $\horizon=\batchsize \numbatches$, otherwise we can take $\horizon:= \left\lfloor \frac{\horizon}{\batchsize} \right\rfloor \batchsize $. Recall that in the batch setting the agent receives the rewards after each batch ends, meaning that the agent operates with the same amount of information within a single batch. For simplicity, we assume that as long as the history remains the same the decision rule does not change as well. Note that this assumption does not impose any significant restrictions. Indeed, instead of applying a policy once, one can always do it $\batchsize$ times until the history updates. Thus, a batch policy is also a finite sequence of $\policy = (\policy_{\timeidx})_{1 \leq \timeidx \leq \horizon}$ of decision rules: $\policy_\timeidx : \historyset_{\timeidx-1} \xrightarrow{} \setofdist(\actset_{\timeidx})$. However, not the whole past history is available for the agent in timestep $\timeidx$, formally, $\history_\timeidx = \history_{\timeidx_{\batchidx}}$ for any $t_j < t \leq t_{j+1}$. 

In practice, sequential batch learning is usually considered as a limitation of the environment. However, for notation convenience, we consider this limitation from a policy perspective, i.e., we assume that it is not the online agent who works with the batch environment, but the batch policy interacts with the online environment. To distinguish between online and batch policies we will denote the last as $\batchpolicy = (\batchpolicy_{\timeidx})_{1 \leq \timeidx \leq \horizon}$.

We now are ready to formulate the goal of the paper formally. That is, given an arbitrary policy $\policy$, we aim to establish upper and lower regret bounds of its batch specification $\batchpolicy$.

\subsection{Preliminaries}
\label{sec:preliminaries}

Before proceeding, we will need to distinguish between ``good" and ``bad" policies on the basis of some properties.  
Accordingly, we first define a binary relation on a set of policies. We say that the decision rule  $\policy_\timeidx = \policy_\timeidx(\cdot|\history_{\timeidx-1})$ is not worse than the decision rule $\policy_\timeidx^{\prime} = \policy_\timeidx^{\prime}(\cdot|\history_{\timeidx-1})$ (and write $\policy_\timeidx \geq \policy_\timeidx^{\prime}$) if the expected reward under $\policy_\timeidx$ is not less than the expected reward under $\policy_\timeidx^{\prime}$:

\begin{equation}
    \label{better policy}
    \sum_{\action \in \actset_{\timeidx}} \langle \theta_*, \action \rangle \policy_\timeidx(\action|\history_{\timeidx-1}) \geq \sum_{\action \in \actset_{\timeidx}} \langle \theta_*, \action \rangle \policy_\timeidx^{\prime}(\action|\history_{\timeidx-1}).
\end{equation}
If $\geq$ can be replaced by $>$, we say that the decision rule $\policy_\timeidx$ is better than the decision rule $\policy_\timeidx^{\prime}$ (and write $\policy_\timeidx > \policy_\timeidx^{\prime}$).


Define the informativeness of history by the number of times the optimal actions were chosen in it. Formally, let $T_{\bestarmidx}(\policy, \timeidx) = \sum_{\timeidxx=1}^{\timeidx} \mathbb{I} \{\Action_\timeidx = \Action_\timeidx^*\}$ be the number of times policy $\policy$ made optimal decisions in history $\history_{\timeidx}$. Then, we require that the decision rule based on a more informative history is at least as good as the decision rule based on a less informative history:

\begin{assumption}[Informativeness]
\label{variance_contractions}
    Let $T_{\bestarmidx}(\policy, \timeidx)$ and $T^{\prime}_{\bestarmidx}(\policy, \timeidx)$ be numbers of times the optimal actions were chosen in histories $\history_{\timeidx}$ and $\history^{\prime}_{\timeidx}$, correspondingly. If  $T_{\bestarmidx}(t) \geq T^{\prime}_{\bestarmidx}(t)$, then $\policy_{\timeidx+1}(\cdot|\history_{\timeidx}) \geq \policy_{\timeidx+1}(\cdot|\history^{\prime}_{\timeidx})$.
\end{assumption}

Next, we assume that policy $\policy = (\policy_{\timeidx})_{1 \leq \timeidx \leq \horizon}$ improves over time if the ``rate" of increasing of the regret decreases.

\begin{assumption}[Subliniarity]
\label{policy_improvement}
    $\frac{\regret_{\horizonn}(\policy)}{\horizonn} > \frac{\regret_{\horizonnn}(\policy)}{\horizonnn}$ for all $\horizonn, \horizonnn$, $1 \leq \horizonn < \horizonnn \leq \horizon$.
\end{assumption}

Finally, we impose a monotonic lower bound on the probability of choosing the optimal action at timestep $\timeidx$. Here we consider two assumptions: instance-independent monotonicity and instance-dependent. That is, in the former case, we assume the existence of the lower bound $f(\timeidx)$, no matter what instance $\theta$ is given; while in the latter case -- the existence of the lower bound depends on a specific instance of actions $\theta$: $f:=f_{\theta}$. Since instance-independent monotonicity requires a bound for every possible instance $\theta$, it is a stricter assumption.

\begin{assumption}[Instance-independent monotonicity]
    \label{monotonicity}
    For any suboptimal action $\action$, $\action \in \actset_{\timeidx}$ and $\action \notin \arg \max_{\actset_{\timeidx}} \langle \theta_*, \action \rangle$, there exists a function $f_{\action}:[0, \infty) \xrightarrow{} [0, 1] $ such that: (i) $f_{\action}$ is nonincreasing; (ii) $\policy_\timeidx(\action|\history_{\timeidx-1}) \leq f_{\action}(\timeidx)$ for all $\timeidx > 0$; and (iii) $f(\timeidx) := 1 - \sum_{\action} f_{\action}(\timeidx)$ is a strictly increasing function for all $ \timeidx > 0$ unless $f(t)=1$. \footnote{Note that from (ii) and (iii) it follows that for all $\timeidx$, $\policy_\timeidx(\bestarmidx|\history_{\timeidx-1}) \geq f(\timeidx)$.}
\end{assumption}

\begin{assumption}[Instance-dependent monotonicity]
    \label{instance_dependent_monotonicity}
    For any suboptimal action $\action$, $\action \in \actset_{\timeidx}$ and $\action \notin \arg \max_{\actset_{\timeidx}} \langle \theta_*, \action \rangle$, there exists a function $f_{\theta, \action}:[0, \infty) \xrightarrow{} [0, 1] $, depending on $\theta$, such that: (i) $f_{\theta, \action}$ is nonincreasing; (ii) $\policy_\timeidx(\action|\history_{\timeidx-1}) \leq f_{\theta, \action}(\timeidx)$ for all $\timeidx > 0$; (iii) $f_{\theta}(t) := 1 - \sum_{\action} f_{\theta, \action}(\timeidx)$ is a strictly increasing function for all $ \timeidx > 0$ unless $f_{\theta}(t)=1$; and (iv) $f_{\theta}(t)$ is nondecreasing in its instance argument in the following sense: $f_{\theta_1}(t) < f_{\theta_2}(t)$ for all $t>0$ if $\min_{\action \in \actset_{\timeidx}} \suboptgap_{\theta_1}(\action) < \min_{\action \in \actset_{\timeidx}} \suboptgap_{\theta_2}(\action)$, where $\suboptgap_{\theta}(\action) = \max_{b \in \actset_{\timeidx}} \langle \theta, b - \action \rangle$ is suboptimality gap.
\end{assumption}

The validity discussion of the imposed assumptions can be found in Section \ref{app:assum2.1} - \ref{app:assum2.4}. 


\section{Batch learning for stochastic linear bandits}
\label{sec:warm-up}

In this section, we provide lower and upper bounds on the best achievable performance for different linear bandit settings, namely: 2-armed bandits, bandits with finitely many arms, and bandits with infinitely many arms.

\subsection{Stochastic linear bandits with 2 arms}

We start with a more restricted analysis of 2-armed problem because, first, it allows to derive a stronger result, and, second, it gives insight into the analysis of more difficult settings.

\begin{theorem}
\label{thm1}
    Let $\pi^b$ be a batch specification of a given policy $\pi$, $\numofactions=2$, and $\numbatches = \frac{\horizon}{\batchsize}$. Suppose that assumptions \ref{variance_contractions} and \ref{policy_improvement} hold. Then, for $b>1$,
    
    \begin{equation}
    \label{main_thm}
        \regret_\horizon(\policy) < \regret_\horizon(\batchpolicy) \leq \batchsize \regret_\numbatches(\policy).
    \end{equation}
\end{theorem}


\begin{proof}
We can consider the term $\batchsize \regret_\numbatches(\policy)$ as $b$ similar online independent agents operating in the same environment by following policy $\policy$ but over a shorter horizon $M$. However, we can also think of it as the performance of a batch policy. Indeed, imagine a na\"ive agent that deliberately repeats each step $b$ times instead of immediately updating its beliefs and proceeding to the next round as in an online manner. Then, after $b$ repetitions (after a batch ends), it updates its beliefs using only the first reward from the previous batch. So, while this policy could perform as an online policy within these repetitions, it pretends that rewards are not observable and acts using the outdated history (just like a batch policy does). For notational simplicity, we refer to this policy as online ``short" policy and denote it by $\dumnpolicy$.

\textbf{Step 1 (Within batch).} Fix $\batchidx \geq 1$. Let $\policy_{\timeidx_\batchidx} \geq \batchpolicy_{\timeidx_\batchidx} \geq \dumnpolicy_{\timeidx_\batchidx}$. Define an average decision rule between timesteps $\timeidx_1$ and $\timeidx_2$ as $\Bar{\policy}_{\timeidx_1, \timeidx_2} = \frac{\sum_{\timeidxx=\timeidx_1}^{\timeidx_2-1} \policy_\timeidxx}{\timeidx_2 - 1 - \timeidx_1 }$; and an average decision rule in batch $\batchidx$ as $\avpolicy{\batchidx} = \frac{\sum_{\timeidxx=\timeidx_\batchidx}^{\timeidx_{\batchidx+1}-1} \policy_\timeidxx}{b }$. By the end of batch $j$, we have:

    \begin{align*}
        \Bar{\policy}_{\timeidx_\batchidx,\timeidx} \stackrel{(a)}{>} \Bar{\policy}_{\timeidx_\batchidx, \timeidx_\batchidx + 1} = \policy_{\timeidx_\batchidx} \geq & \batchpolicy_{\timeidx_\batchidx} \stackrel{(b)}{=} \batchpolicy_{\timeidx} \text{,  and} \\
        & \batchpolicy_{\timeidx_\batchidx} \geq \dumnpolicy_{\timeidx_\batchidx} \stackrel{(c)}{=} \dumnpolicy_{\timeidx},
    \end{align*}
    for any timestep $\timeidx_{\batchidx} < \timeidx < \timeidx_{\batchidx+1}$, where $(a)$ follows from Lemma \ref{lemma_propeties} (\ref{point1}); and $(b)$ and $(c)$ hold by the definition of batch policy. Thus, starting with $\policy_{\timeidx_\batchidx} \geq \batchpolicy_{\timeidx_\batchidx} \geq \dumnpolicy_{\timeidx_\batchidx}$ at the beginning of batch $\batchidx$ leads us to $\avpolicy{\batchidx} > \avbatchpolicy{\batchidx} \geq \avdumnpolicy{\batchidx}$. Moreover, by Lemma \ref{lemma_propeties} (\ref{point2}), we have $\policy_{\timeidx_{\batchidx + 1} - 1} > \batchpolicy_{\timeidx_{\batchidx + 1} - 1} \geq \dumnpolicy_{\timeidx_{\batchidx + 1} - 1}$.  
    
    \textbf{Step 2 (Between batches).}
    Fix $\batchidx \geq 1$. Let $\avpolicy{{l}} > \avbatchpolicy{{l}} \geq \avdumnpolicy{{l}}$ for any batch $1 \leq l < j$. Let $\history_{\timeidx_{\batchidx}-1}$, $\history^\batchsize_{\timeidx_{\batchidx}-1}$, $\history^{\prime}_{\timeidx_{\batchidx}-1}$ be histories collected by policies $\policy$, $\batchpolicy$, $\dumnpolicy$ by the timestep $\timeidx_{\batchidx}$, correspondingly. Let $\bestarmidx = \arg\max_\action \langle \theta_*, \action \rangle$ be an optimal action. \footnote{Since we have just 2 actions, we automatically assume that the optimal action is independent of $\timeidx$.} Define a number of times we have received a reward from action $\action$ \footnote{Usually, it is defined as  a number of times action $\action$ has been played but, since the policy $\dumnpolicy$ plays more actions than receives rewards, we define it that way.} in batch $\batchidx$ by policy $\policy$ as $T_a(\policy, \batchidx) = \sum_{\timeidxx=\timeidx_\batchidx}^{\timeidx_{\batchidx+1}-1} \mathbb{I} \{\Action_\timeidx = \action\}$. Note that $\mathbb{E} [T_a(\policy, \batchidx)] =  b \cdot \avpolicy{{\batchidx}}(a) $. Since $\numofactions = 2$, $\avpolicy{{l}} > \avbatchpolicy{{l}} \geq \avdumnpolicy{{l}}$ implies $\avpolicy{{l}}(\bestarmidx) > \avbatchpolicy{{l}}(\bestarmidx) \geq \avdumnpolicy{{l}}(\bestarmidx)$ for $1 \leq l < j$. Hence, $\mathbb{E}[T_{\bestarmidx}(\policy, l)] > \mathbb{E}[T_{\bestarmidx}(\batchpolicy, l)] \geq \mathbb{E}[T_{\bestarmidx}(\dumnpolicy, l)]$ for $1 \leq l < j$ and, therefore, $\sum_l \mathbb{E}[T_{\bestarmidx}(\policy, l)] > \sum_l \mathbb{E}[T_{\bestarmidx}(\batchpolicy, l)] \geq \sum_l \mathbb{E}[T_{\bestarmidx}(\dumnpolicy, l)]$. By applying Assumption \ref{variance_contractions}, we have that $\policy_{\timeidx_{\batchidx}} > \batchpolicy_{\timeidx_{\batchidx}} \geq \dumnpolicy_{\timeidx_{\batchidx}}$.
    
    \textbf{Step 3 (Regret throughout the horizon).}
    We assume that the interaction begins with the online policy, batch policy, and ``short" online policy being equal to each other: $\policy_{1} = \batchpolicy_{1} = \dumnpolicy_{1}$. Then, from Step 1, by the end of the first batch, we have $\avpolicy{1} > \avbatchpolicy{1} \geq \avdumnpolicy{1}$ and $\policy_{\timeidx_{2} - 1} > \batchpolicy_{\timeidx_{2} - 1} \geq \dumnpolicy_{\timeidx_{2} - 1}$. Next, from Step 2, the transition to the second batch retains the relation between policies: $\policy_{\timeidx_{2}} > \batchpolicy_{\timeidx_{2}} \geq \dumnpolicy_{\timeidx_{2}}$; and so on. Finally, summing over $\numbatches=\frac{\horizon}{\batchsize}$ batches, we have:
    
    \begin{align*}
        \regret_\horizon(\policy) & = \mathbb{E} \left [ \sum_{\timeidx=1}^{\horizon} \langle \theta_*, \action^* \rangle - \totalreward_{\horizon} \right ] = \mathbb{E} \left [ \sum_{\timeidx} \langle \theta_*, \action^* - \Action_{\timeidx} \rangle \right ] \\
        & = \sum_{\timeidx} \sum_{\action} \langle \theta_*, \action^* - \action \rangle \policy_{\timeidx}(\action) \\
        & = \batchsize \sum_{\batchidx=1}^{\numbatches} \sum_{\action} \langle \theta_*, \action^* - \action \rangle \avpolicy{{j}}(\action) \\
        & < \batchsize \sum_{\batchidx=1}^{\numbatches} \sum_{\action} \langle \theta_*, \action^* - a \rangle \avbatchpolicy{{j}}(a) = \regret_\horizon(\batchpolicy) \\
        & \leq \batchsize \sum_{\batchidx=1}^{\numbatches} \sum_{\action} \langle \theta_*, \action^* - a \rangle \avdumnpolicy{{j}}(a) = \batchsize \regret_\numbatches(\policy).
    \end{align*}
\end{proof}

\subsection{Stochastic linear bandits with finitely many arms}

As we mentioned in the proof of Theorem \ref{thm1}, in the case of $\numofactions=2$, it immediately follows that the more often a decision rule chooses the optimal action, the better it is. In contrast, this is not generally true when $\numofactions>2$. Indeed, some decision rules might value the optimal and the worst actions so that another decision rule that puts less weight on optimal action is better (because it chooses interim actions more often at the expense of worst action). Consequently, the case of $\numofactions>2$ introduces a new subtlety and requires a number of additional steps in the analysis. Specifically, in order to derive result similar to Theorem \ref{thm1}, we utilize Assumption \ref{instance_dependent_monotonicity} and introduce a meta-algorithm (Algrotihm \ref{alg2:dealayed_batch_learn}).

From Assumption \ref{instance_dependent_monotonicity}, it follows that after timestep $\tau_{\theta_*} := \min\{t: f_{\theta_*}(t) > 1/K \}$ policy $\policy$ in the worst-case scenario: (i) behaves better than the policy that acts uniformly at random; and (ii) acts more optimally within each consequent timestep. As such, while batch and online ``short" policies ($\batchpolicy$ and $\dumnpolicy$) remain the same, online policy $\policy$ gets better in the worst case and, therefore, chooses optimal action more often, providing lower regret. Thus, it would be enough to split the horizon $\horizon$ into two phases: before and after timestep $\tau_{\theta_*}$, and ensure that all three policy specifications behave similarly during the first phase. That is to say, some na\"ive policy $\policy^0$ operates independently of history $\history_{\timeidx}$ during phase 1. 

However, $\tau_{\theta_*}$ depends on vector of rewards $\theta_*$, which is unknown for policy $\policy$, and, as a consequence, we do not know where we should stop phase 1 to bring the reasoning above to life. But what if one could find an estimate $\hat{\tau}$ of true $\tau_{\theta_*}$ such that $f_{\theta_*}(\hat{\tau})$ is greater than $1/\numofactions$ with high probability for some confidence level $\delta$ (i.e., $\mathbb{P}(f_{\theta_*}(\hat{\tau}) > 1/\numofactions)>1-\delta$)? Then we could split the original horizon $\horizon$ into two parts (phases) and: (i) run a uniform random policy in phase 1 (before timestep $\hat{\tau}$); and (ii) apply policy $\policy$ in phase 2 (after timestep $\hat{\tau}$).  

To capture the logic above, we adjust the learning process, which can be represented as a meta-algorithm built upon a given policy $\policy$ (see Algorithm \ref{alg2:dealayed_batch_learn}). As the name suggests, while performing randomly during phase 1, Algorithm \ref{alg2:dealayed_batch_learn} keeps all three specifications' regrets similar and, therefore, allows formulating the following result.

\begin{algorithm}[tb]
    \caption{Approximate learning with delayed start}
    \label{alg2:dealayed_batch_learn}
    \begin{algorithmic}
        \STATE {\bfseries Input:} horizon $\horizon \geq 0$, candidate policy $\policy$, number of actions $\numofactions$, monotonic lower bound $f_{\cdot}$, confidence level $\delta$
        \STATE $\pi^0 \gets Unif(K)$ 
        \STATE $phase1 \gets \texttt{False}$
        \STATE $\timeidx \gets 1$
        \STATE $\history_0 \gets \emptyset$
        \REPEAT
            \STATE $\Action_\timeidx \gets \policy^0(\cdot)$ 
            \STATE $\history_{\timeidx} \gets$ \textsc{UpdateHistory}
                \STATE $phase1 \gets$ \textsc{CheckPhase}($\history_{\timeidx}$, $\delta$) 
            \STATE $\timeidx \gets \timeidx+1$
        \UNTIL{($phase1$ is \texttt{False} {\bfseries and} $t \notin \grid$ ) {\bfseries or} $\timeidx \leq \horizon$}
        
        \REPEAT 
            \STATE $\Action_\timeidx \gets \policy(\cdot|\history_{\timeidx-1})$ 
            \STATE $\history_{\timeidx} \gets$ \textsc{UpdateHistory}
            \STATE $\timeidx \gets \timeidx+1$
        \UNTIL{$\timeidx \leq \horizon$}
    \end{algorithmic}
\end{algorithm}

\begin{theorem}
    \label{main_thm2:long}
    Let $\pi^b$ be a batch specification of a given policy $\pi$, $\numofactions < \infty$, $\numbatches = \frac{\horizon}{\batchsize}$, and \textsc{CheckPhase} be with the failure probability $\delta$ (Theorem \ref{thm:check_phase}). Suppose that assumptions \ref{variance_contractions}, \ref{policy_improvement}, and \ref{instance_dependent_monotonicity} hold. Also, assume that policy $\Tilde{\policy}$ represents a policy constructed following Algorithm \ref{alg2:dealayed_batch_learn} for a given candidate policy $\policy$. Then, for $b>1$,
    
    \begin{align*}
        \underline{\regret_\horizon(\Tilde{\policy})} < & \underline{\regret_\horizon(\Tilde{\policy}^{\batchsize})} \leq \batchsize \underline{ \regret_\numbatches(\Tilde{\policy})}
    \end{align*}
    with probability $1-\delta$, where $\underline{\regret_\horizon(\policy)}$ is the worst-case regret.
\end{theorem}

\begin{proof}
    First, we split the total regret into two terms (representing phase 1 and phase 2) and notice that the first term is the same for both policies ($\Tilde{\policy}$ and $\Tilde{\policy}^{\batchsize}$), as the same uniform random policy $\policy^0$ operates during the first phase.
    
    \begin{align}
        \regret_\horizon(\Tilde{\policy}) & = \regret_{1:\hat{\tau}-1}(\policy^0) + \regret_{\hat{\tau}:n}(\policy), \label{eq:pol_approx1}\\
        \regret_\horizon(\Tilde{\policy}^{\batchsize}) & = \regret_{1:\hat{\tau}-1}(\policy^0) + \regret_{\hat{\tau}:n}(\batchpolicy) \label{eq:pol_approx2},
    \end{align}
    where $\regret_{1:\timeidxx}(\policy) = \mathbb{E} \left [ \sum_{\timeidx=1}^{\timeidxx} \max_{\action \in \actset_{\timeidx}} \langle \theta_*, \action \rangle - \totalreward_{\horizon} \right ]$.
    Note that Algorithm \ref{alg2:dealayed_batch_learn} is implemented in such a way that phase 1 can only be ended with the end of the batch and therefore:
    
    \begin{align*}
        \regret_\horizon(\Tilde{\policy}^{\prime})  & = \regret_{1:\hat{\tau}-1}(\policy^0) + \regret_{\hat{\tau}:n}(\dumnpolicy)  \\
        & = \regret_{1:\hat{\tau}-1}(\policy^0) + \batchsize \regret_{\hat{\tau}:n}(\policy) = \batchsize \regret_\numbatches(\Tilde{\policy}).
    \end{align*}
    
    Next, we express the second term as a sum of instantaneous regrets and exploit Assumption \ref{instance_dependent_monotonicity}:
    
    \begin{align*}
        \regret_{\hat{\tau}:n}(\policy) & = \mathbb{E} \left ( \sum_t \langle \theta_*, \Action_{\timeidx}^* - \Action_{\timeidx} \rangle \right ) \\
        & = \mathbb{E} \left ( \sum_t \sum_{\action} \langle \theta_*, \Action_{\timeidx}^* - \action \rangle \mathbb{I}\{ \Action_{\timeidx}=\action \} \right ) \\
        & = \sum_t \sum_{\action} \langle \theta_*, \Action_{\timeidx}^* - \action \rangle \mathbb{P} \left ( \Action_{\timeidx}=\action \right | \history_{\timeidx-1} ) \\
        & = \sum_t \sum_{\action} \langle \theta_*, \Action_{\timeidx}^* - \action \rangle \policy_{\timeidx} \left ( \action \right | \history_{\timeidx-1} ) \\
        & \leq \sum_t \sum_{\action \neq \Action_{\timeidx}^*} \langle \theta_*, \Action_{\timeidx}^* - \action \rangle f_{\theta, \action} \left ( \timeidx \right ) = \underline{\regret_{\hat{\tau}:n}(\policy)}.
    \end{align*}
    
    Next we reproduce Theorem \ref{thm1} for $\underline{\regret_{\hat{\tau}:n}(\policy)}$, $\underline{\regret_{\hat{\tau}:n}(\batchpolicy)}$, and $\underline{\regret_{\hat{\tau}:n}(\dumnpolicy)}$. Specifically, Step 1 remains the same, as it only involves the general notion of decision rules. Step 3 retains the same semantic and is only to be rewritten in a general setting; the only semantic change is expected in Step 2, as we need to release the assumption of 2 actions. Note that in this part, we abuse notation a bit and by $\policy_{\timeidx}, \batchpolicy_{\timeidx}, \dumnpolicy_{\timeidx}, \avpolicy{\batchidx}, \avbatchpolicy{\batchidx}, \avdumnpolicy{\batchidx}$ we assume its lower and upper bounds provided by Assumption \ref{instance_dependent_monotonicity}.
    
    \textbf{Step 1 (Within batch).} See Step 1 from Theorem \ref{thm1}.  
    
    \textbf{Step 2 (Between batches).}
    Fix $\batchidx \geq 1$. Let $\avpolicy{{l}} > \avbatchpolicy{{l}} \geq \avdumnpolicy{{l}}$ for any batch $1 \leq l < j$. Let $\history_{\timeidx_{\batchidx}-1}$, $\history^\batchsize_{\timeidx_{\batchidx}-1}$, $\history^{\prime}_{\timeidx_{\batchidx}-1}$ be histories collected by policies $\policy$, $\batchpolicy$, $\dumnpolicy$ by the timestep $\timeidx_{\batchidx}$, correspondingly. Let $\Action_{\timeidx}^* = \arg \max_{\actset_{\timeidx}} \langle \theta_*, \action \rangle$ be an optimal action for round $\timeidx$, and, unlike Theorem \ref{thm1}, denote $T_{\bestarmidx}(\policy, \batchidx) = \sum_{\timeidxx=\timeidx_\batchidx}^{\timeidx_{\batchidx+1}-1} \mathbb{I} \{\Action_\timeidx = \Action_\timeidx^*\}$ -- number of times policy $\policy$ made optimal choice in batch $\batchidx$.
    
    Recall that $f_{\hat{\theta}}$ is strictly increasing with probability $1-\delta$ for all $\timeidx > \hat{\tau}$. As a consequence, using the fact that: (i) $f_{\hat{\theta}, \action}$ is nonincreasing for any suboptimal action $\action$ and (ii) $f_{\hat{\theta}}$ is increasing with high probability, $\avpolicy{{l}} > \avbatchpolicy{{l}} \geq \avdumnpolicy{{l}}$ implies $\avpolicy{{l}}(\Action_{\timeidx}^*) > \avbatchpolicy{{l}}(\Action_{\timeidx}^*) \geq \avdumnpolicy{{l}}(\Action_{\timeidx}^*)$ for $1 \leq l < j$ and for $0<\timeidx<\timeidx_\batchidx$ with probability $1 - \delta$. Hence, $\mathbb{E}[T_{\bestarmidx}(\policy, l)] > \mathbb{E}[T_{\bestarmidx}(\batchpolicy, l)] \geq \mathbb{E}[T_{\bestarmidx}(\dumnpolicy, l)]$ for $1 \leq l < j$ and, therefore, $\sum_l \mathbb{E}[T_{\bestarmidx}(\policy, l)] > \sum_l \mathbb{E}[T_{\bestarmidx}(\batchpolicy, l)] \geq \sum_l \mathbb{E}[T_{\bestarmidx}(\dumnpolicy, l)]$. By applying Assumption \ref{variance_contractions}, we have that $\policy_{\timeidx_{\batchidx}} > \batchpolicy_{\timeidx_{\batchidx}} \geq \dumnpolicy_{\timeidx_{\batchidx}}$.
    
    \textbf{Step 3 (Regret throughout the horizon).} Applying Step 3 from Theorem \ref{thm1} to $\underline{\regret_{\hat{\tau}:n}(\policy)}$, $\underline{\regret_{\hat{\tau}:n}(\batchpolicy)}$, and $\underline{\regret_{\hat{\tau}:n}(\dumnpolicy)}$ gives:
    
    \begin{align}
        \underline{\regret_{\hat{\tau}:n}(\policy)} < \underline{\regret_{\hat{\tau}:n}(\batchpolicy)} \leq \batchsize \underline{\regret_{\hat{\tau}:n}(\dumnpolicy)} \label{eq:main_res_wcr}.
    \end{align}
    Putting \eqref{eq:pol_approx1}-\eqref{eq:pol_approx2} and \eqref{eq:main_res_wcr} together, with probability $1-\delta$, we get:
    
    \begin{align*}
        \underline{\regret_\horizon(\Tilde{\policy})} & := \regret_{1:\hat{\tau}-1}(\policy^0) + \underline{\regret_{\hat{\tau}:n}(\policy)} \\
        & < \regret_{1:\hat{\tau}-1}(\policy^0) + \underline{\regret_{\hat{\tau}:n}(\batchpolicy)} =: \underline{\regret_\horizon(\Tilde{\policy}^{\batchsize})} \\
        & \leq \regret_{1:\hat{\tau}-1}(\policy^0) + \batchsize \underline{\regret_{\hat{\tau}:n}(\policy)} =: \batchsize \underline{\regret_\numbatches(\Tilde{\policy})}.
    \end{align*}
    
\end{proof}

\textbf{\textsc{CheckPhase} procedure.} Since function $f_{\theta}$ decreases in its instance argument in a certain sense (point (iv), Assumption \ref{instance_dependent_monotonicity}), it suffices to find a more ``difficult" environment with instance $\hat{\theta}$ (and corresponding $\hat{\tau} := \min\{t: f_{\hat{\theta}}(t) > 1/K \}$), such that $\mathbb{P}(f_{\theta_*}(\hat{\tau}) > 1/K)>1-\delta$). Intuitively, the smaller the difference between a suboptimal and the optimal actions ($\suboptgap_{\theta}(\action) = \max_{b \in \actset_{\timeidx}} \langle \theta, b - \action \rangle$), the more difficult it is to distinguish the optimal action and, therefore, the probability of choosing the optimal action should be smaller. Imagine now that after timestep $\timeidx$, we have some estimate $\tilde{\theta}_\timeidx$ of the unknown parameter vector $\theta_*$  and confidence set $\mathcal{C}_t$ that contains $\theta_*$ with high probability. Then, we can choose $\hat{\theta}$ that underestimates the true reward along the current best action and overestimates the true rewards along all other actions, making the minimal suboptimal gap smaller. Algorithm \ref{alg3:check_phase} formalizes this logic.

\begin{algorithm}[tb]
    \caption{\textsc{CheckPhase}}
    \label{alg3:check_phase}
    \begin{algorithmic}
        \STATE {\bfseries Input:} history $\history_{\timeidx}$, confidence level $\delta$
        \STATE $phase1 = \texttt{True}$
        \STATE $i = \arg \max_{j} \hat{\mu}_j(t)$
        \STATE $LCB_{i,t} = \hat{\mu}_i(t) - \sqrt{\frac{\ln t}{T_i(t)}}$
        \STATE $l = \arg \max_{j \neq i} \hat{\mu}_j(t)$
        \FOR{ $l \neq i$}
            \STATE $UCB_{l,t} = \hat{\mu}_l(t) + \sqrt{\frac{\ln t}{T_l(t)}}$
        \ENDFOR
        \STATE $\hat{\theta} = \left ( LCB_{i,t}, \bigcup_{l \neq i} UCB_{l,t} \right )$
        
        \IF{$f_{\hat{\theta}}(t) > 1/\numofactions$ {\bfseries and} $2 \numofactions / t^2 < \delta$}
        \STATE $phase1 = \texttt{False}$
        \ENDIF
        \STATE {\bfseries return} $phase1$
    
    
    \end{algorithmic}
\end{algorithm}

\begin{theorem}
\label{thm:check_phase}
    \textsc{CheckPhase} procedure presented in Algorithm \ref{alg3:check_phase} is with the failure probability $\delta$.
\end{theorem}

\begin{proof}
    Since $\numofactions < \infty$, we can compose an auxiliary reward vector $(\mu_i)_{1 \leq i \leq \numofactions}$, where $\mu_{i} = \langle \theta_*, \action_i \rangle$ is true reward of action $\action_i$. For simplicity, we consider the case $K=2$ first. Without loss of generality, assume that action 1 is optimal. Following Algorithm \ref{alg3:check_phase}, at timestep $\timeidx$ the agent computes estimates $\hat{\mu}_i(t)$ and their confidence intervals $LCB_{i,t}$, $UCB_{i,t}$ for $i \in [1,2]$ and then constructs a new instance with lower bound for action 1 and upper bound for action 2:
    
    \begin{equation*}
        \hat{\theta} = \left ( LCB_{1,t}, UCB_{2,t} \right ).
    \end{equation*}
    There may be two situations, that we will consider separately:
    
    \begin{enumerate}
        \item Confidence intervals do not cover true rewards for actions $1$ or $2$: $\mu_1 \notin [ LCB_{1,t},  UCB_{1,t}]$ or $\mu_2 \notin [ LCB_{2,t},  UCB_{2,t}]$;
        \item Confidence intervals cover true rewards for actions $1$ and $2$: $\mu_1 \in [ LCB_{1,t}, UCB_{1,t}]$ and $\mu_2 \in [LCB_{2,t}, UCB_{2,t}]$.
    \end{enumerate}
    
    First note that under option 2 the following holds
    \begin{equation}
    \label{eq:more_dif_env_stoch_bandit}
        LCB_{\bestarmidx_{\theta_*},t} - UCB_{\action,t} < \armval_{\bestarmidx_{\theta_*}} - \armval_{\action}
    \end{equation}
    for all $\action \neq \bestarmidx_{\theta_*}$, where $\bestarmidx_{\theta_*}$ is the best action in the original environment, $\bestarmidx_{\theta_*} = \max_{\action \in \actset} \langle \theta_*, \action \rangle$. By \eqref{eq:more_dif_env_stoch_bandit} we have:
    
    \begin{align}
    \label{eq:more_dif_env}
        \langle \hat{\theta}, \bestarmidx_{\theta_*} - \action \rangle & = \langle \hat{\theta}, \bestarmidx_{\theta_*} \rangle - \langle \hat{\theta}, \action \rangle = LCB_{\bestarmidx_{\theta_*},t} - UCB_{\action,t} \\
        & \stackrel{\eqref{eq:more_dif_env_stoch_bandit}}{<} \armval_{\bestarmidx_{\theta_*}} - \armval  = \langle \theta_*, \bestarmidx_{\theta_*} \rangle - \langle \theta_*, \action \rangle = \langle \theta_*, \bestarmidx_{\theta_*} - \action \rangle \nonumber
    \end{align}
    for all $\action \neq \bestarmidx_{\theta_*}$. Therefore, due to the point (iv) of Assumption \ref{instance_dependent_monotonicity}, $f_{\hat{\theta}}(t) < f_{\theta_*}(t)$ for all $t>0$. As such, $\hat{\tau} := \min\{t: f_{\hat{\theta}}(t) > 1/K \}$ is greater than true $\tau$ and, therefore, is a valid estimate of the beginning of $phase2$ in Algorithm \ref{alg2:dealayed_batch_learn}.
    
    According to Hoeffding's inequality, the first option arises with probability at most $ \delta := 4 /t^2$. Thus, $\mathbb{P}(f_{\theta_*}(\hat{\tau}) > 1/K)>1-\delta$.
    
    To extend this result to the case $\numofactions>2$, one needs to note that we are interested in pairwise comparisons of the best action $\bestarmidx$ and a suboptimal action $\action$. Thus, reproducing the logic above for each action $\action \neq \bestarmidx$ and assigning $ \delta := 2 \numofactions /t^2$, one can observe the same result for an arbitrary number of actions $\numofactions$.
\end{proof}

\subsection{Stochastic linear bandits with infinitely many arms}
\label{sec:main_res}
So far, we have focused on the stochastic linear bandit problem with finite number of actions $\numofactions$. We now provide an analysis for the most general setting -- stochastic linear bandit problem with $\actset_{\timeidx} \subset \mathbb{R}^d$. Intuitively, if $|\actset_{\timeidx}| = \infty$ (e.g., $\actset_{\timeidx} = \{ \action : \norm{\action}_2 \leq \timeidx \}$), there is the unique optimal action $\Action_\timeidx^*$ for each timestep $\timeidx$ and any instance $\theta_*$, but no second best action. In that sense, each environment instance would be of the same ``difficulty'' for the agent. As such, we leverage a more strict assumption of monotonicity (Assumption \ref{monotonicity}) and provide instance-independent analysis in this section.

Similarly to the intuition prior to Theorem \ref{main_thm2:long}, after timestep $\tau := \min\{t: f(t) > 0 \}$ policy $\policy$ in the worst-case: (i) chooses the optimal action with non-zero probability; and (ii) acts more optimally within each consequent timestep. Except in this case, we do not need to estimate $\tau$, as it is given by stronger Assumption \ref{monotonicity}. Thus, it is enough to split the horizon $\horizon$ into two phases: before and after timestep $\tau$, and ensure that all three policy specifications behave similarly during the first phase. That is to say, some na\"ive policy $\policy^0$ operates independently of history $\history_{\timeidx}$ during phase 1.

As previously, to capture the logic above, we present a meta-algorithm built upon a given policy $\policy$ (see Algorithm \ref{alg1:dealayed_batch_learn}). Except in this case, we replace the random uniform policy $\pi_0$ with some na\"ive policy, that would keep the regrets during the phase 1 at the same level.

\begin{algorithm}[tb]
   \caption{Learning with delayed start}
   \label{alg1:dealayed_batch_learn}
\begin{algorithmic}
    \STATE {\bfseries Input:} horizon $\horizon \geq 0$, candidate policy $\policy$, na\"ive policy $\policy^0$, monotonic lower bound $f$
    \STATE $\timeidx \gets 1$
    \STATE $\history_0 \gets \emptyset$
    \STATE $\tau = \min_{\batchidx} \{ \timeidx_{\batchidx}: f(\timeidx_{\batchidx}) > 0  \}$
    \REPEAT
        \STATE $\Action_\timeidx \gets \policy^0(\cdot)$ 
        \STATE $\history_{\timeidx} \gets$ \textsc{UpdateHistory}
        
        \STATE $\timeidx \gets \timeidx+1$
    \UNTIL{$\timeidx < min\{\tau, \horizon \}$}
    
    \REPEAT 
        \STATE $\Action_\timeidx \gets \policy(\cdot|\history_{\timeidx-1})$ 
        \STATE $\history_{\timeidx} \gets$ \textsc{UpdateHistory}
        \STATE $\timeidx \gets \timeidx+1$
    \UNTIL{$\timeidx \leq \horizon$}
\end{algorithmic}
\end{algorithm}

\begin{theorem}
\label{main_thm:long}
    Let $\pi^b$ be a batch specification of a given policy $\pi$, $\numbatches = \frac{\horizon}{\batchsize}$. Suppose that assumptions \ref{variance_contractions}, \ref{policy_improvement}, and \ref{monotonicity} hold. Also, assume that policy $\Tilde{\policy}$ represents a policy constructed following Algorithm \ref{alg1:dealayed_batch_learn} for a given candidate policy $\policy$. Then, for $b>1$,
    
    \begin{align*}
        \underline{\regret_\horizon(\Tilde{\policy})} < & \underline{\regret_\horizon(\Tilde{\policy}^{\batchsize})} \leq \batchsize \underline{ \regret_\numbatches(\Tilde{\policy})},
    \end{align*}
   where $\underline{\regret_\horizon(\policy)}$ is the worst case-regret.
\end{theorem}

\begin{proof}
    The proof of the theorem immediately follows from the Theorem \ref{main_thm2:long} by replacing $\hat{\tau}$ and $f_{\theta}$ with $\tau$ and $f$ correspondingly.
\end{proof}

\section{Supplementary Discussions}
\label{app:discussion}
\subsection{Intuition behind Assumption \ref{variance_contractions}}
\label{app:assum2.1}

The exploration-exploitation trade-off is a fundamental dilemma between choosing an action that reveals the most information and provides the highest immediate reward. In a nutshell, policies that satisfy Assumption \ref{variance_contractions} are good in leveraging exploration and exploitation decisions in the following sense: 
\begin{itemize}
    \item if the policy gets to a situation when the optimal action has a smaller confidence interval on average (i.e., it chose the optimal action more often) , and the past choices were reasonable (i.e., history $\history_{\timeidx}$ brought either the highest reward or the most information), then this policy will make a better consequent decision;
    
    \item if another history $\history_{\timeidx}^{\prime}$ were fed to the policy, that would imply not optimal leveraging of exploration and exploitation decisions and, as a result, lead to (i) fewer choices of optimal actions (or higher uncertainty); and (ii) worse subsequent decisions.
\end{itemize}
 
In other words, a policy with history $\history_{\timeidx}$ will be ahead of any different history $\history_{\timeidx}^{\prime}$ from exploration-exploitation perspective. And in combination with Assumption \ref{policy_improvement} that would lead to better subsequent decisions.

\subsection{Intuition behind Assumption \ref{policy_improvement}}
We provide some properties of a policy with sublinear regret that we use in the proof of the main result.

\begin{lemma}
\label{lemma_propeties}
Let $\policy = ( \policy_\timeidx )_{1 \leq \timeidx \leq \horizon}$ be a  policy such that Assumption \ref{policy_improvement} holds. Then,
    \begin{enumerate}
        \item 
        \label{point1}
        $\Bar{\policy}_{\horizonnn} > \Bar{\policy}_{\horizonn}$, where $\Bar{\policy}_{\timeidx} = \frac{\sum_{\timeidxx=1}^{\timeidx} \policy_\timeidxx}{\timeidx}$ is an average decision rule;
        \item 
        \label{point2}
        $\policy_{\timeidx} > \Bar{\policy}_{\timeidx}$ $\forall \timeidx$ such that $1 \leq \timeidx \leq \horizon$,
    \end{enumerate}
where $\policy_\timeidx + \policy_\timeidxx$ is an elementwise addition of two probability vectors for some $\timeidx, \timeidxx$.
\end{lemma}

\begin{proof}
\parindent=0pt
    1. First, we need to show that $\Bar{\policy}_{\timeidx}$ is a decision rule for some $\timeidx$, i.e., $\sum_{\action \in \actset_{\timeidx}} \Bar{\policy}_{\timeidx}(\action) = 1$ and $\Bar{\policy}_{\timeidx}(\action) \geq 0$ for all $\action \in \actset_{\timeidx}$. Indeed,
    \begin{align*}
        \sum_{\action \in \actset_{\timeidx}} \Bar{\policy}_{\timeidx}(\action) & = \sum_{\action \in \actset_{\timeidx}} \frac{\sum_{\timeidxx=1}^{\timeidx} \policy_\timeidxx (\action)}{\timeidx} = \frac{\sum_{\timeidxx=1}^{\timeidx} \sum_{\action \in \actset_{\timeidx}} \policy_\timeidxx (\action)}{\timeidx} = \frac{\sum_{\timeidxx=1}^{\timeidx} 1}{\timeidx} = 1.
    \end{align*}
    
    Since $\policy_\timeidxx (\action) \geq 0$ for all $\action \in \actset_{\timeidx}$ and for all $1 \leq \timeidxx \leq \timeidx $, $\Bar{\policy}_{\timeidx}(\action) \geq 0$ for all $\action \in \actset_{\timeidx}$. Next, we convert $\mathbb{E} [ \totalreward_{\horizon} ]$ into the sum over timesteps and actions:
    \begin{align*}
        \mathbb{E}  [ \totalreward_{\horizon}  ] & = \mathbb{E} \left[ \sum_\timeidx \reward_\timeidx \right] = \mathbb{E} \left[ \sum_{\timeidx} \sum_{\action} \reward_\timeidx \mathbb{I} \{ \Action_\timeidx = \action \} \right] \\
        & = \mathbb{E} \left[ \sum_\timeidx \sum_\action \mathbb{E} [ \reward_\timeidx \mathbb{I} \{ \Action_\timeidx = \action \} | \Action_\timeidx ] \right] \\
        & = \mathbb{E} \left[ \sum_\timeidx \sum_\action \langle \theta_*, \action \rangle \mathbb{I} \{ \Action_\timeidx = \action \} \right] \\
        & = \sum_\timeidx \sum_a \langle \theta_*, \action \rangle \mathbb{P}_{\env, \policy} ( \Action_\timeidx = \action ) \\
        & = \sum_\timeidx \sum_\action \langle \theta_*, \action \rangle \policy_\timeidx ( \action | \history_{\timeidx-1} ).
    \end{align*}
    
    Fix $\horizonn, \horizonnn: \horizonn < \horizonnn$. From Assumption \ref{policy_improvement}, we have 
    $\frac{\regret_{\horizonn}(\policy)}{\horizonn} > \frac{\regret_{\horizonnn}(\policy)}{\horizonnn}$. Expressing the regret by its definition, one can get 
    
    \begin{equation*}
        \frac{\mathbb{E} \left [ \sum_{\timeidx=1}^{\horizonn} \max_{\action \in \actset_{\timeidx}} \langle \theta_*, \action \rangle - \totalreward_{\horizonn} \right ]}{\horizonn} > \frac{\mathbb{E} \left [ \sum_{\timeidx=1}^{\horizonnn} \max_{\action \in \actset_{\timeidx}} \langle \theta_*, \action \rangle - \totalreward_{\horizonnn} \right ]}{\horizonnn},
    \end{equation*}
    
    and hence $\frac{\mathbb{E} [\totalreward_{\horizonnn}]}{\horizonnn} - \frac{\mathbb{E} [\totalreward_{\horizonn}]}{\horizonn} > 0$.
    
    Finally,
    
    \begin{align*}
        & \frac{\mathbb{E} [\totalreward_{\horizonnn}]}{\horizonnn} - \frac{\mathbb{E} [\totalreward_{\horizonn}]}{\horizonn} \\
        = & \frac{\sum_{\timeidx=1}^{\horizonnn} \sum_\action \langle \theta_*, \action \rangle \policy_\timeidx ( \action | \history_{\timeidx-1} )}{\horizonnn} - \frac{\sum_{\timeidx=1}^{\horizonn} \sum_\action \langle \theta_*, \action \rangle \policy_\timeidx ( \action | \history_{\timeidx-1} )}{\horizonn} > 0.
    \end{align*}
    
    The result is completed by rearranging the sums and using the definition of $\Bar{\policy}_{\horizonn}, \Bar{\policy}_{\horizonnn}$.
    
    2. For $\timeidx<\horizon$ we have $ \frac{ \regret_{\timeidx}(\policy)}{\timeidx} > \frac{ \regret_{\timeidx+1}(\policy)}{\timeidx+1} $. By subtracting $\frac{ \regret_{\timeidx+1}(\policy)}{\timeidx}$ from both sides we get:
    \begin{align*}
        \frac{ \regret_{\timeidx}(\policy) - \regret_{\timeidx+1}(\policy)}{t} & > \frac{\timeidx \regret_{\timeidx+1}(\policy) - (\timeidx+1) \regret_{\timeidx+1}(\policy)}{\timeidx(\timeidx+1)}, \\
        \frac{-(\max_{\action \in \actset_{\timeidx}} \langle \theta_*, \action \rangle - \reward_{\timeidx+1})}{\timeidx} & > \frac{-\regret_{\timeidx+1}(\policy)}{(\timeidx+1)\timeidx}, \\
        \max_{\action \in \actset_{\timeidx}} \langle \theta_*, \action \rangle - \reward_{\timeidx+1} & < \frac{\regret_{\timeidx+1}(\policy)}{\timeidx+1}, \\
        \sum_\action \langle \theta_*, \action \rangle \policy_{\timeidx+1}(\action) & > \sum_\action \langle \theta_*, \action \rangle \Bar{\policy}_{\timeidx+1}(\action).
    \end{align*}
    
    Here, in forth step we used that $\frac{\sum_{\action \in \actset_{\timeidx}}\sum_{\timeidxx=1}^{\timeidx+1} \policy_\timeidxx (\action)}{\timeidx+1} = \sum_{\action \in \actset_{\timeidx}} \Bar{\policy}_{\timeidx+1}(\action)$.
\end{proof}

In what follows, we show that Assumption \ref{policy_improvement} is essential because if it does not hold, then the lower bounds in Theorems \ref{thm1}, \ref{main_thm:long}, and \ref{main_thm2:long} is greater than the upper bound: $\regret_\horizon(\policy) > \batchsize \regret_\numbatches(\policy).$

\begin{lemma}
    Suppose that the opposite to Assumption \ref{policy_improvement} holds for a given policy $\policy$. Then
    \begin{equation*}
        \regret_\horizon(\policy) > \batchsize \regret_\numbatches(\policy).
    \end{equation*}
\end{lemma}

\begin{proof}
    Suppose Assumption \ref{policy_improvement} does not hold for $\policy$ (e.g., it makes a lot of suboptimal choices). In that case, an online ``short" policy could perform better as it omits these suboptimal choices. Indeed, using Assumption \ref{policy_improvement} for horizons $\numbatches$ and $\horizon$, we have 
    
    \begin{equation*}
        \frac{\regret_\numbatches(\policy)}{\numbatches} \geq \frac{\regret_\horizon(\policy)}{\horizon},
    \end{equation*}
    
    and, therefore, multiplying by $\horizon$, and $\numbatches$ respectively and using $\horizon = \numbatches \batchsize$, we get 
    
    \begin{equation*}
        \batchsize \numbatches \regret_\numbatches(\policy)  > \numbatches \regret_\horizon(\policy).
    \end{equation*}
    
    Finally, dividing by $\numbatches$, we get 
    \begin{equation*}
        \batchsize \regret_\numbatches(\policy) > \regret_\horizon(\policy).
    \end{equation*}
\end{proof}

\subsection{Equivalence of Assumptions \ref{policy_improvement} and \ref{monotonicity}}
Here, we discuss a link between Assumptions \ref{policy_improvement} and \ref{monotonicity}. Generally, neither of the assumptions implies another. Nevertheless, there is a strong connection between these assumptions. 

On the one hand, Assumption \ref{monotonicity} may seem more restrictive, as it imposes the strictly increasing lower bound on the probability of choosing the optimal action, whereas Assumption \ref{policy_improvement} is satisfied as long as the ”rate” of increasing of the total regret decreases (which does not necessarily require the strictly increasing probability of choosing the optimal action over time). In contrast, Assumption \ref{monotonicity} only requires increasing lower bound, while \ref{policy_improvement} requires the decreasing rate of the actual total regret. In fact, the strictly increasing lower bound on the probability of choosing the optimal action (Assumption \ref{monotonicity}) is equivalent to the decreasing rate of the worst-case total regret (adjusted Assumption \ref{policy_improvement}):

\begin{align*}
    f \text{ -- strictly increasing } \Leftrightarrow \left( \underline{\regret_{\horizonn}(\policy)} \right ) / \horizonn > \left( \underline{\regret_{\horizonnn}(\policy)} \right ) / \horizonnn.
\end{align*}

Moreover, these assumptions are equivalent under certain conditions, summarized in the following lemma.

\begin{lemma}
    Consider a stochastic bandit problem with $\numofactions=2$ arms, then Assumptions \ref{policy_improvement} and \ref{monotonicity} are equivalent.
\end{lemma}

\begin{proof}
    Recall that in 2-armed stochastic bandit $\theta_* = (\armval_\action)_{\action \in [1,2]}$ - a vector of true rewards for each action (arm). Without loss of generality, assume that arm 1 is optimal. Define a number of times policy $\policy$ has played an arm $\action$ by timestep $\timeidx$ as $T_{\action}(\timeidx) = \sum_{\timeidxx=1}^{\timeidx} \mathbb{I} \{\Action_\timeidx = \action\}$. Using the regret decomposition lemma (Lemma 4.5, \cite{lattimore_szepesvari_2020}), we can derive:
    
    \begin{align*}
        \frac{\regret_{\horizonn}(\policy)}{\horizonn} > \frac{\regret_{\horizonnn}(\policy)}{\horizonnn} \Leftrightarrow \frac{ \mathbb{E} \left [ T_{2}(\horizonn) \right ] }{\horizonn} > \frac{ \mathbb{E} \left [ T_{2}(\horizonnn) \right ] }{\horizonnn}
    \end{align*}
    
    From frequentist probability perspective, $\frac{ \mathbb{E} \left [ T_{2}(\horizonn) \right ] }{\horizonn}$ can be interpreted as probability of choosing arm 2: $\mathbb{P} \left ( \Action_{\horizonn} = 2 \right )$. As such, $f_{2}(t):=\mathbb{P} \left ( \Action_{\timeidx} = 2 \right )$ is an decreasing function and $f(t):=1 - f_{2}(t)$ is an increasing function. Clearly, $\policy_{\timeidx}(2 | \history_{\timeidx-1}) \leq f_{2}(t)$. Thus, all the conditions of Assumption \ref{monotonicity} are satisfied.
\end{proof}

While everything we discussed here holds for Assumption \ref{instance_dependent_monotonicity} as well, we provide a more concrete discussion of Assumption \ref{instance_dependent_monotonicity} in the next section.

\subsection{Illustration of Assumptions \ref{instance_dependent_monotonicity}}
\label{app:assum2.4}
Here, we provide a specific example of a policy that satisfies Assumption \ref{instance_dependent_monotonicity}. We consider the stochastic multi-armed bandit
problem and examine the UCB family of algorithms. 

\begin{lemma}[Lemma 1.2, \cite{Lecture_shipra}]
    Let $T_{\action}(\timeidx)$ be the number action $\action$ is chosen by UCB algorithm run on instance $\theta_* = (\armval_1,...,\armval_{\numofactions})$ of the stochastic multi-armed bandit problem. Then, for any action $\action \neq \arg \max_{\action} \armval_\action$, 
    
    \begin{equation*}
        \mathbb{E} \left [ T_{\action}(\timeidx) \right ] \leq \frac{4 \ln{\timeidx}}{\Delta^2_{\action}} + 8,
    \end{equation*}
    
    where $\Delta_{\action} = \max_{\action} \armval_\action - \armval_{\action}$.
\end{lemma}

From frequentist probability perspective, we can think of $\frac{\mathbb{E} \left  [ T_{\action}(\timeidx) \right ]}{t}$ as probability of choosing arm $\action$ and, thus, we can set $f_{\theta_*, \action}(t) = \frac{4 \ln(\timeidx+1)}{t\Delta^2_{\action}} + 8/t$ which is a nondecreasing upper bound of the probability of choosing suboptimal arm $\action$ (i.e., points (i) and (ii) are satisfied). As Assumption \ref{instance_dependent_monotonicity} suggests, define $f_{\theta_*}$ as $f_{\theta_*, \action}(t) = 1 - \sum_{\action \neq \arg \max \armval_\action} f_{\theta_*, \action}(t)$. Note that $f_{\theta_*}(t)$ satisfies points (iii) and (iv) of Assumption \ref{instance_dependent_monotonicity}, as (iii) $f_{\theta_*, \action}(t)$ is actually decreasing function for all $\action$ and all $\timeidx>0$; and (iv) $f_{\theta_*, \action}(t)$ is increasing in the instance argument (as described in Assumption \ref{instance_dependent_monotonicity}). Thus, we get a monotonical lower bound of the probability of choosing the optimal arm.

\section{Empirical analysis}
\label{empirical}
Our theoretical results posit that algorithms \ref{alg2:dealayed_batch_learn} and \ref{alg1:dealayed_batch_learn} are run optimally, i.e, utilizing monotonic lower bound $f$, which can be indeed achieved (see Sections \ref{app:assum2.1} - \ref{app:assum2.4}). However in our experiments it suffices to use a general batch learning specification described in Section \ref{sec:batch}. In fact, the analysis we conduct in this section provides even stronger as the bounds depicted in figures \ref{fig:regret}, \ref{cr-plot} hold for the actual values rather than the worst-case values.

We perform experiments on two different applications: simulated stochastic bandit environments; and a contextual bandit environment, learned from the logged data in an online marketing campaign. We examine the effect of batch learning on Thompson Sampling (TS) and Upper Confidence Bound (UCB) policies for the stochastic problems, and linear Thompson Sampling (LinTS) \citep{NIPS2011_e53a0a29} and linear Upper Confidence Bound (LinUCB) \citep{LinUCB_2010} for the contextual problem.

\textbf{Simulated environments.} We present some simulation results for the Bernoulli bandit problem. In this simulation, the best action has a reward probability of $0.5$ and $\numofactions - 1$ have  a probability of $0.5 - \Delta$. In total, we consider six environments for $K \in \{2,5,10\}$ and $\Delta \in \{ 0.1, 0.02 \}$.  Figure \ref{fig:regret} shows the regret as a function of the batch size $\batchsize$ for various settings and policies.

\textbf{Real data.} 
We also consider batch learning in a marketing campaign on the logged dataset from our industrial partner. The company 
has recently used three different campaigns to sell an extra broadband subscription to their customers. In the current dataset, a randomly selected set of customers received randomly one of the three campaigns. The data contains a sample of a campaign selection from October 2019 until June 2020 combined with customer information.  We adopt an unbiased offline evaluation method \citep{Li_2011} to compare various bandit algorithms and batch size values. We use conversion rate (CR) as the metric of interest, defined as the ratio between the number of successful interactions and the total number of interactions. To protect business-sensitive information, we only report relative conversion rate; therefore, Figure \ref{cr-plot} demonstrates the CR returned by the off-policy evaluation algorithm relatively to the online performance.

\textbf{Results.} \footnote{The source code of the experiments can be found in \url{https://github.com/danilprov/batch-bandits}.}
Figure \ref{fig:regret} represents the effect of batch size across three dimensions: number of action $\numofactions$, suboptimal gap value $\Delta$, and the policy. As expected, the regret has an upward trend as batch size increases for all settings. Taking a look at the environment parameters, $\numofactions$ and $\Delta$, we see that the lower the difficulty of the environment (i.e., the higher the suboptimality gap or number of actions), the stronger the impact of batching.

While the exact behavior of the batch learning depends majorly on the exact policy, we still can see a clear difference between randomized policies (TS and LinTS) and deterministic ones (UCB and LinUCB). \footnote{A randomized policy returns a distribution over actions as an output, and the action to take is then to be chosen from this distribution. A deterministic policy returns a degenerate distribution over actions, i.e., an action to take is defined deterministically.} Indeed, figure \ref{fig:regret} shows that TS is much more robust to the impact of batching, whereas the UCB algorithm suffers notable deterioration. The results for the real data also confirm this fact: from Figure \ref{cr-plot} we observe that the impact of batching is milder for LinTS than for LinUCB in the contextual problem.

It is important to note that both experiments demonstrate results consistent with the theoretical analysis conducted in Sections \ref{sec:warm-up}. As the upper bound in Theorem \ref{thm1} suggests, the performance metric (i) reacts evenly to the increasing/decreasing batch size and (ii) doesn't violate the imposed bounds.

\begin{figure}%
\centering
\subfigure[][TS, $\Delta=0.1$]{%
\label{fig:ex3-a}%
\includegraphics[height=1.17in]{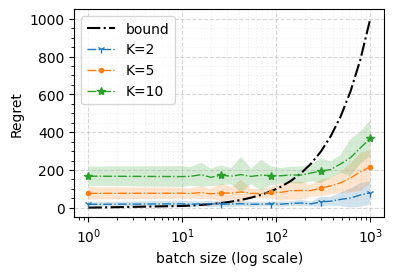}}%
\hspace{8pt}%
\subfigure[][UCB, $\Delta=0.1$]{%
\label{fig:ex3-b}%
\includegraphics[height=1.17in]{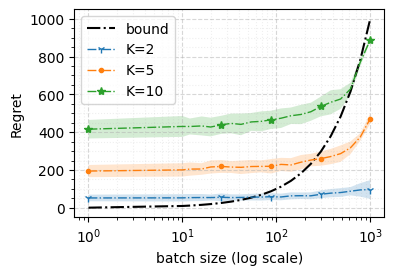}}
\\
\subfigure[][TS, $\Delta=0.02$]{%
\label{fig:ex3-c}%
\includegraphics[height=1.2in]{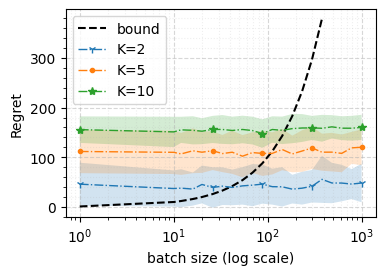}}%
\hspace{8pt}%
\subfigure[][UCB, $\Delta=0.02$]{%
\label{fig:ex3-d}%
\includegraphics[height=1.2in]{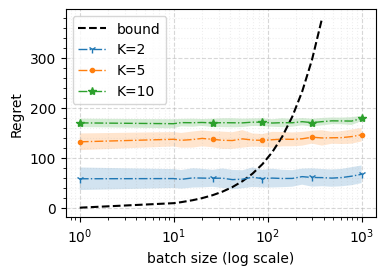}}%
\caption[A set of four subfigures.]{Empirical regret performance by TS and UCB policies by batches for $K \in \{2,5,10\}$ and $\Delta \in \{ 0.1, 0.02 \}$. The plots are averaged over 100 repetitions. The black line is the upper bound from \eqref{main_thm}  shifted such that it goes through the origin.}%
\label{fig:regret}%
\end{figure}

\begin{figure}%
\centering
\subfigure[][LinTS]{%
\label{fig:lints}%
\includegraphics[height=1.17in]{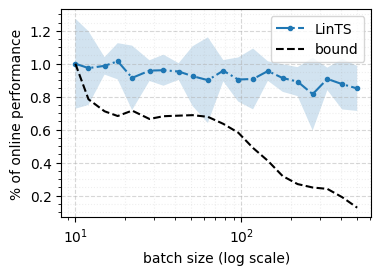}}%
\hspace{8pt}%
\subfigure[][LinUCB]{%
\label{fig:linucb}%
\includegraphics[height=1.17in]{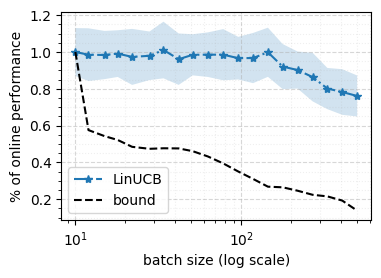}}
\caption[Caption for CR]{Empirical conversion rate of LinTS and LinUCB policies by batches for the real dataset. The plots are averaged over 20 repetitions. The black line is the lower \footnotemark bound from \eqref{main_thm}  normalized relatively to the online behavior.}
\label{cr-plot}
\end{figure}

\footnotetext{Note, since the regret is unknown for the real life data as it measures the performance relatively to the oracle behavior, we plot reward (conversion rate). As such, the upper bound from \eqref{main_thm} for regret  becomes the lower bound for reward.}

\section{Conclusion}
\label{sec:conclusion}

We have presented a systematic approach for batch learning in stochastic linear bandits. The contribution of this paper is twofold. First, we have introduced a new perspective of batched bandits emphasizing the importance of the batch size effect. In contrast, most of the work on batched bandits assumes that the batch size is a parameter to be optimized. Second, we have shown the actual effect of batch learning by conducting a comprehensive theoretical analysis, which is confirmed by our strong empirical results. Practically speaking, we have investigated one component of the performance-computational cost trade-off and demonstrated that it deteriorates gradually depending on the batch size. Thus, to find a suitable batch size, practitioners should take the necessary steps to estimate the second component (engineering costs) based on computational capabilities.



An interesting direction concerning future work would be to consider batch learning with varying decision rules within batches. That is, even though the agent does not receive any information within a batch, it might leverage two kinds of information to make more intelligent decisions. The first and natural recourse of information is the actual experience from the past batches. The second source of information is simulated experience: experience simulated by a model of the environment and combined with the agent's behavior in the current batch. Studying how to integrate actual knowledge and simulated experience to perform more optimally within a batch would potentially lead to tighter bounds between batch and online learning.

\section{Acknowledgements}
This project is partially financed by the Dutch Research Council (NWO) and the ICAI initiative in collaboration with KPN, the Netherlands.

\newpage
{\footnotesize
\bibliography{main}}
\bibliographystyle{IEEEtran}

\end{document}